\documentclass[10pt]{article}
\usepackage[utf8]{inputenc}
\usepackage[margin=1in]{geometry}
\usepackage{rotating}
\usepackage{booktabs}
\usepackage{array}
\usepackage{amsmath, amsfonts}
\usepackage{amssymb}
\DeclareMathOperator*{\argmax}{arg\,max}

\usepackage{color}
\usepackage{hyperref}
\usepackage{bbm}
\usepackage{algorithm}
\usepackage[noend]{algpseudocode}
\usepackage{subcaption}
\usepackage{graphicx,nicefrac}
\usepackage{tikz}
\usepackage{adjustbox}
\usepackage{framed}
\usepackage[most]{tcolorbox}

\usepackage{mathtools}
\mathtoolsset{showonlyrefs}

\usepackage[
    backend=biber,
    style=ieee,
    citestyle=numeric-comp,
    sorting=none,
    giveninits=true,
    natbib,
    hyperref,
    maxbibnames=99,
    doi=true,isbn=false,url=true,eprint=false
]{biblatex}

\usepackage{amsthm}
\usepackage[normalem]{ulem}
\usepackage{soul}
\usepackage{tabularx}
\newtheorem{theorem}{Theorem}

\newtheorem{cor}{Corollary}
\newtheorem{definition}{Definition}

\usepackage{chngcntr}
\usepackage{apptools}
\AtAppendix{\counterwithin{theorem}{section}}
\AtAppendix{\counterwithin{prop}{section}}

\usepackage{mathtools}

\def\C{\mathcal{C}}
\def\D{\mathcal{D}}
\def\L{\mathcal{L}}

\def\Q{\mathcal{Q}}
\def\R{\mathbb{R}}
\def\S{\mathcal{S}}
\def\T{\mathcal{T}}
\def\V{\mathcal{V}}
\def\X{\mathcal{X}}
\def\Y{\mathcal{Y}}
\def\O{\mathcal{O}}
\newcommand{\tOPT}{\mathrm{OPT}}

\newcommand\ddfrac[2]{\frac{\displaystyle #1}{\displaystyle #2}}

\makeatletter
\def\blfootnote{\xdef\@thefnmark{}\@footnotetext}
\makeatother

\title{Optimal Data Selection: An Online Distributed View}

\begin{document}
\author{Mariel Werner$^1$ \and Anastasios Angelopoulos$^2$ \and Stephen Bates$^{1,2}$ \and Michael Jordan$^{1,2}$}
\date{
$^1$Department of Electrical Engineering and Computer Sciences, \\ University of California, Berkeley\\
$^2$Department of Statistics, University of California, Berkeley
}
\maketitle

\begin{abstract}
The blessing of ubiquitous data also comes with a curse: the communication, storage, and labeling of massive, mostly redundant datasets. We seek to solve this problem at its core, collecting only valuable data and throwing out the rest via submodular maximization. Specifically, we develop algorithms for the online and distributed version of the problem, where data selection occurs in an uncoordinated fashion across multiple data streams. We design a general and flexible core selection routine for our algorithms which, given any stream of data, any assessment of its value, and any formulation of its selection cost, extracts the most valuable subset of the stream up to a constant factor while using minimal memory. Notably, our methods have the same theoretical guarantees as their offline counterparts, and, as far as we know, provide the first guarantees for \emph{online distributed} submodular optimization in the literature.  Finally, in learning tasks on ImageNet and MNIST, we show that our selection methods outperform baselines in settings of interest by $5-20\%$.
\end{abstract}

\section{Introduction}\label{sec:intro}
It is often infeasible to store and label large, real-world datasets in their entirety due to constraints on memory capacity, annotation resources, and communication bandwidth. Consequently, methods for extracting high-value subsets from large collections of data that incur low selection and memory cost are increasingly indispensable. Within a learning framework, one sensible approach is to assign high value to those subsets of the training data which improve a classifier's performance at test time. 
In practice, however, identifying such high-value training sets is challenging. Real-world datasets are often highly redundant with few truly consequential instances (e.g., sequential frames in a video), such that labels are costly to obtain and of little marginal value. Furthermore, when data is sufficiently plentiful, it may be necessary to filter incoming data in a distributed fashion across several machines or edge devices, so that only the most valuable subset gets sent to a central server. 
In this paper, we address these challenges using online distributed submodular optimization.

Before generalizing to the distributed setting, we preview the way our algorithms select data on a single data stream. Let $\D=\{x_1,...,x_n\} \in \X^n$ denote a sequence of $n$ points, where $\X$ is a feature space, and let $f$ be a function that takes as input a subset of $\D$ and outputs some notion of its desirability, or \emph{value}, as a positive real number.\footnote{For ease of expression, we use set syntax for our sequential framework: by `subset' we mean subsequence, by $A \subseteq B$ we mean that $A$ is a subsequence of sequence $B$, and by $2^A$ we mean the set of all subsequences of sequence $A$.} Points from $\D$ arrive sequentially and we must make a decision at each time step to select the current point or not. Our goal is to select a subset $\L \subseteq \D$ whose value is close to the same-sized subset of $\D$ that has optimal value

\begin{equation}\label{eq:opt-objective}
    \tOPT(f,\D,|\L|),
\end{equation}
where for a general $k \in \mathbb{N}$ we define
\begin{equation}\label{eq:card-constrained-opt}
    \tOPT(f,\D,k)
    \triangleq
    \argmax_{S \subseteq \D:|S|=k}f(S).
\end{equation}

\begin{table*}[t]
    \caption{Threshold schedule suitable for desired objective}
    \begin{tabularx}{\textwidth}[ht!]{| 
    >{\centering\arraybackslash}X| >{\centering\arraybackslash}X|}
    \hline
    \textbf{Practical Objective} & \textbf{Threshold Schedule} \\
    \hline
    Ensure that selected set has high value amongst sets of same cardinality. & 
    $\tau_t=\gamma$, where $\gamma$ is a fixed constant (gives maximal $\nicefrac{1}{2}$-approx. in \eqref{eq:dmgt-intro-bound}).\\
    \hline
    Only select points if their value outweighs their marginal selection cost under cost function $c$. & $\tau_t = c(\L_t \cup \{x_t\}) - c(\L_t)$\\
    \hline
    If training a model on selected points, improve accuracy of model early on and save labeling costs later on, selecting only the most valuable points. & Increase $\tau_t$ with $t$.\\
    \hline
    Select set of size at most $k$, the selection budget. & See \cite{badanidiyuru2014} and related work \cite{agrawal2018,kazemi2019,liu2021,nikolakaki2021,norouzi-fard2018} for different threshold schedules which achieve this objective. \\
    \hline
    \end{tabularx}
    \label{table: threshold schedules}
\end{table*}

\begin{table*}[t]
    \caption{Submodular value function suitable for desired objective (see Supplement Section 3 for proofs of submodularity)}
    \centering
    \begin{tabularx}{\textwidth}[ht!]{| >{\centering\arraybackslash}X| >{\centering\arraybackslash}X|}
    \hline
    \textbf{Practical Objective} & \textbf{Submodular Value Function} \\
    \hline
    Select class-balanced substream from class-imbalanced stream using a model's ($\hat{\pi}$) predictions. &
    $f_{\hat{\pi}}(\L \cup \{x\}) =
    \sum_{k \in [K]} {\hat{\pi}}_k(x) g\left(\sum_{x \in \L \cup \{x\}} {\hat{\pi}}_k(x)\right)$ where $g$ is any concave function.\\
    \hline
    Select points from the stream which are similar to some target set $\Q$. For instance, if we're selecting a training set for a classifier from a stream and want the classifier to perform very well on certain classes, $\Q$ could be a representative set of those classes. &
    In the following examples, $s$ is a similarity metric.
    \newline
    \newline
    Facility Location: $f(\L)=\sum_{y \in \Q}\max_{x \in \L} s(x,y)$ 
    \newline
    \newline
    Graph Cut: $f(\L)=\sum_{x \in \L, y \in \Q}s(x,y)$ 
    \\
    \hline
    \end{tabularx}
    \label{table: value functions}
\end{table*}

Finding $\tOPT$ generally requires solving an intractable combinatorial problem. We use a dynamic-thresholding routine that finds a high-value approximation to $\tOPT$ without storing $\D$, simply by selecting a each point in the stream if its marginal value exceeds a data-dependent threshold. This procedure, which we call \emph{dynamic marginal gain thresholding} (\texttt{DMGT} --- see Algorithm \ref{alg:dmgt}), is the core routine of our distributed selection algorithms. Formally, at time step $t$, given the currently selected set $\L_{t-1}$, \texttt{DMGT} selects the current point $x_t$ if
\begin{equation}\label{eq:sel-rule}
    f(\L_{t-1} \cup \{x_t\}) - f(\L_{t-1})
    >
    \tau_t,
\end{equation}
where $\tau_t$ is a data-dependent threshold set by any protocol that the analyst desires. Under this selection rule, given any submodular, nonnegative, monotone increasing value function $f$ and schedule of thresholds, $\{\tau_t\}_{t \in [|\D|]}$, \texttt{DMGT} returns a set $\L \subseteq \D$ such that
\begin{equation}\label{eq:dmgt-intro-bound}
    f(\L) 
    \geq 
    \ddfrac{\tau_{\min}}{\tau_{\min} + \tau_{\max}} f(\tOPT(f,\D,|\L|)),
\end{equation}
where $\tau_{\min}$ and $\tau_{\max}$ are the minimum and maximum elements of $\{\tau_t\}_{t \in |\D|}$ respectively. In words, \texttt{DMGT} selects a set whose value is \emph{at least} a $\nicefrac{\tau_{\min}}{\tau_{\min}+\tau_{\max}}$-approximation to the value of the optimal, same-sized set. The pre-factor in \eqref{eq:dmgt-intro-bound} is maximal at $\nicefrac{1}{2}$ by choosing $\tau_{\min}=\tau_{\max}$, but in many practical cases, a non-uniform schedule might be preferable (see Table \ref{table: threshold schedules} for some examples). Our bound gives a guarantee for any choice of threshold schedule. Tables \ref{table: threshold schedules} and \ref{table: value functions} illustrate a few examples of many different possible pairings between practical objectives, threshold schedules, and value functions for which our method gives guarantees.

\paragraph{Distributed Setting.} The generality of our method (namely that for any choice of $\{\tau_t\}_{t \in |\D|}$, the threshold-dependent bound in \eqref{eq:dmgt-intro-bound} holds) is useful in distributed settings. Oftentimes, it's infeasible to store large amounts of data at a central site for post-processing, necessitating that data selection be distributed across multiple agents or devices. Regardless of agents' different selection resources, our distributed online algorithms still return a set whose value is close to the value of the optimal same-sized subset from the agents' pooled streams. Specifically, we achieve the same style of $O(\nicefrac{1}{\min(M,k)})$-approximation observed in distributed offline submodular optimization \cite{mirzasoleiman2013}, where $M$ is the number of agents and $k$ is the selection budget.

\subsection{Related work}\label{sec:rel_work}
Solving \eqref{eq:card-constrained-opt} effectively and efficiently is a longstanding problem in submodular optimization since finding an exact solution is NP-hard for most choices of submodular $f$ \cite[]{feige1998}. For monotone increasing, nonnegative, and submodular $f$, and a cardinality constraint on feasible sets, \cite{nemhauser1978} proposed a greedy $(1-\nicefrac{1}{e})$-approximation, followed by computationally efficient variants \cite[]{minoux1978,vondrak2014,kumar2015,mirzasoleiman2015}. These algorithms are non-streaming and require that $\D$ be maintained in memory. In the streaming setting, \cite{badanidiyuru2014} achieved a $(\nicefrac{1}{2} - \epsilon)$-approximation for any $\epsilon > 0$ in $O(\nicefrac{k \log k}{\epsilon})$ memory (where $k$ is an upper bound on cardinality of feasible subsets). Subsequently, \cite{kazemi2019} achieved the same approximation using only $O(\nicefrac{k}{\epsilon})$ memory. For $\D$ random rather than arbitrary, \cite{norouzi-fard2018} achieved a greater-than-$\nicefrac{1}{2}$-approximation in expectation, \cite{agrawal2018} achieved a $(1-\nicefrac{1}{e}-\epsilon)$-approximation in $O(k2^{\text{poly}(\nicefrac{1}{\epsilon})})$ memory, and most recently \cite{liu2021} achieved the same constant-factor approximation in $O(\nicefrac{k}{\epsilon})$ memory. Motivated by applications to team formation, \cite{nikolakaki2021} introduced a streaming algorithm which approximates an unconstrained, cost-diminished version of \eqref{eq:card-constrained-opt} for arbitrary $\D$ and submodular, nonnegative $f$ in $O(k)$ memory. 
Threshold-based selection algorithms have been known and used in streaming submodular optimization for a long time. However, our algorithm gives constant-factor approximations for a wider class of thresholds.

To our knowledge, we provide the first theoretical guarantees and proposals for \emph{online distributed} submodular maximization algorithms. \cite{clark2014,golovin2010} study distributed online submodular maximization for sensor selection and resource-constrained networks, but allow points to be selected multiple times to learn an unknown value function over rounds. In our work, we do not permit repeated selection of points. Parallelization of offline distributed submodular maximization algorithms across different machines using the MapReduce framework is the focus of \cite{barbosa2015,kumar2015,mirrokni2015,mirzasoleiman2015,barbosa2016,mirzasoleiman2016,liu2018}. Close to our setting of interest is the multi-agent/machine regime in \cite{mirzasoleiman2013}. They show that for offline, $k$-cardinality-constrained optimization of a submodular, nonnegative objective with $M$ agents, it's possible to achieve a $O(\nicefrac{1}{\min(M,k)})$ constant approximation to the value of the optimal subset from the agents' pooled original sets. Even though we restrict to the streaming setting in this work, we achieve the same type scaling in $M$ and $k$ as in the offline setting.

We also draw inspiration from the literature on submodular optimization \cite[]{bach2013} for real-world applications, such as representation, summarization, and network modeling [\cite{kempe2003,dueck2007,leskovec2007,agrawal2009,elarini2009,das2011,elarini2011,bilmes2011,lin2011,das2012,rodriguez2012,sipos2012,dasgupta2013,tschiatschek2014,zheng2014,bairi2015}]. Of particular interest are \cite{kothawade2021, kaushal2021} who introduced a general framework for non-streaming active learning with submodular information measures \cite[]{iyer2021,iyer2015}. They focused on correcting, with active labeling, real-world cases of class imbalance, presence of out-of-distribution samples, and redundancy in unlabeled data sets.

\subsection{Our contribution}\label{sec:contr}
\begin{enumerate}
    \item  We propose two algorithms for online distributed data selection whose optimality guarantees match those found in offline distributed selection.
    \item We design a general and flexible threshold-based selection procedure as the core of our algorithms that (i)
    given any arbitrary (and possibly adversarial) data stream $\D$, (ii) any submodular, nonnegative, monotone increasing value function $f$, and (iii) any set of thresholds $\{\tau_t\}_{t \in |\D|}$ chosen by the analyst, returns a subset $\L \subseteq \D$ whose value is at least a fraction, $\nicefrac{\tau_{\min}}{\tau_{\min} + \tau_{\max}}$, of the value of the optimal same-sized set. This procedure has no memory requirements beyond what is needed to store $\L$, and it gives a lower bound on the value of $\L$ in terms of the thresholds $\{\tau_t\}_{t \in |\D|}$. Our theoretical analysis gives constant-factor guarantees for a broad set of selection algorithms, substantially enlarging the space of labeling rules that the analyst can deploy with confidence. 
\end{enumerate}

\section{Theory and Methods}\label{sec:methods}
\subsection{Submodularity}\label{sec:submod}
Let $\D = \{x_1,...,x_n\} \in \X^n$ be a sequence of data, with $\X$ a feature space. Our objective is to extract in an online fashion a high-value subset, $\L$, of $\D$ at low cost, where $f:2^{\D} \rightarrow \R^+$ quantifies the value of subsets of $\D$.
The \emph{value} of a set of points depends on the application or the objectives of the analyst. If the objective is to extract a class-balanced subset, then $f$ would assign high value to class-balanced subsets of the data stream $\D$. If the objective is to maximize diversity in selected sets, such that the data are separated in the feature space, subsets of $\D$ with this property have high value. Our method selects each incoming point if its \emph{marginal gain} under $f$, given the currently selected set, exceeds a designated threshold at that time. 
\begin{definition}[Marginal gain]\label{def:marg-gain}
The \emph{marginal gain} under function $g$ of set $S$ given set $T$ is
\begin{equation}\label{eq:marg-gain}
    g(S|T)
    \triangleq
    g(S \cup T) - g(T).
\end{equation}
\end{definition}
Regardless of how large the already selected set is, if a new point adds sufficient value, it will be included. However, for most reasonable value functions, adding a point to a set should always count more than adding it to a superset of that set. \emph{Submodular} functions \cite[e.g.,][]{bach2013,iyer2015} have this diminishing returns property. 
\begin{definition}[Submodularity]\label{def:submod}
A set function $f:2^{\X} \rightarrow \R$ is \emph{submodular} if for all sets $S \subseteq T \subseteq \X$ and $x \in \X \backslash T$, we have
\begin{equation}\label{eq:submod-def}
    f(S \cup \{x\}) - f(S)
    \geq
    f(T \cup \{x\}) - f(T).
\end{equation}
\end{definition}

The diminishing returns property of the value function regulates selection cost by rejecting points that don't contribute sufficient value to the already-selected set.

\subsection{Distributed Dynamic Marginal Gain Thresholding}\label{sec:dis-dmgt}
In this section we present our two online distributed selection algorithms, \texttt{Distributed DMGT} (Alg. \ref{alg:dist-dmgt}) and \texttt{Distributed DMGT w/ Filtering} (Alg. \ref{alg:dist-dmgt-filt}), along with their core selection routine, \texttt{DMGT}. We first show that the value of pooled sets selected by \texttt{Distributed DMGT} is $O(\nicefrac{1}{M})$-close to the optimal value. We next show that, with further selection by a central agent, \texttt{Distributed DMGT w/ Filtering} tightens this approximation to $O(\nicefrac{1}{\min(M,|\L|)})$, where $\L$ is the set selected by the central agent.

We begin with \texttt{Distributed DMGT}. The algorithm is simple: at every time step, each agent receives a point from their stream of data and compares its marginal value given the currently selected set to a threshold value. This threshold value is allowed to depend on all previous threshold values and all previous data points, including the current point. The agent then selects the point if its marginal value exceeds the corresponding threshold. Finally, the routine returns the combined selected sets of the agents.

We denote a general threshold schedule corresponding to data stream $\D$ by $\{\tau_t\}_{t \in |\D|} \in \R^{|\D|}$, where $\tau_t$ may depend on any information observed in the algorithm by the analyst prior to time step $t$. Formally, $\tau_t = \mathcal{A}(x_1,...,x_t, \tau_1,...,\tau_{t-1})$, where $\mathcal{A}$ is a threshold-setting routine that, at the current time step, is allowed to depend on the current point, all prior points in the stream, and all prior thresholds, and outputs a threshold value for the current time step. 

\begin{algorithm}
\caption{\texttt{Distributed-DMGT}: Distributed Dynamic Marginal Gain Thresholding}
\label{alg:dist-dmgt}
\textbf{Input} $M$ streams of data $\{\D^j\}_{j\in[M]}$, with $\D^j \subset \X$; value function $f:2^{\D} \rightarrow \R^{+}$ (where $\D = \cup_{j \in [M]} \D^j$)
\begin{algorithmic}[1]
\For {$j = 1,...,M$}
\State $\L^j \leftarrow \texttt{DMGT}(\D^j,f)$
\EndFor
\State {\bfseries return} $\L = \cup_{j\in[M]} \L^j$.
\end{algorithmic}
\end{algorithm}

\begin{algorithm}
\caption{\texttt{DMGT}: Dynamic Marginal Gain Thresholding}
\label{alg:dmgt}
\textbf{Input} Stream of data $\D\subset{\X}$; value function $f:2^\D \rightarrow \R^+$
\begin{algorithmic}[1]
\State Set of currently selected points
$\L_0=\emptyset$
\For {$t=1,...,|\D|$}
\State Receive point $x_t$.
\State Set $\tau_t = \mathcal{A}(x_1,...,x_t,\tau_1,...,\tau_{t-1})$.
\If {$f(\L_{t-1} \cup \{x_t\}) - f(\L_{t-1}) > \tau_t$}\label{line:dmgt-sel-rule}
\State $\L_{t} \leftarrow \L_{t-1} \cup \{x_t\}$
\Else
\State $\L_t \leftarrow \L_{t-1}$
\EndIf
\EndFor
\State \Return $\L = \L_t$
\end{algorithmic}
\end{algorithm}

\begin{theorem}[Near optimality of \texttt{Distributed DMGT}]\label{thm:dist-dmgt}
Suppose \text{\texttt{Distributed DMGT}} runs on $M$ separate data streams $\{\D^j\}_{j \in [M]}$, each with submodular, nonnegative, monotone increasing value function $f$. Then
\begin{align}\label{eq:dist-dmgt-bound}
    f(\L) 
    \geq
    \underbrace{\frac{\tau_{\min}}{M(\tau_{\min} + \tau_{\max})}f(\tOPT(f,\D,|\L|))}_{\text{constant-factor optimal}}
     + 
    \underbrace{\frac{\tau_{\min}\tau_{\max}}{M(\tau_{\min} + \tau_{\max})} |\tOPT(f,\D,|\L|) \cap \L|}_{\text{$\geq 0$}},
\end{align}
where $\tau_{\min}$ and $\tau_{\max}$
are the minimum and maximum elements across all the agents' threshold schedules.
\end{theorem}
\begin{proof}
    See \ref{proof:dist-dmgt}.
\end{proof}

The first term in \eqref{eq:dist-dmgt-bound} is the important one: it bounds the value of $\L$ in terms of the value of the same-sized optimal set. The second term, a relic of the proof, is always positive and only improves the optimality of $f(\L)$. \texttt{Distributed DMGT} returns a set whose value is a constant-factor approximation to the value of the optimal set. The approximation factor depends on the number of agents and the minimum and maximum values observed in the agents' combined threshold sets. This is an ex-post-facto guarantee on the selected set, and our method can provide it for any threshold schedule, with varying degrees of optimality. The practitioner can always use prior-knowledge about the data to set a threshold schedule which will yield a higher-value set.

We note two things: 
\\\\
1. The results in Theorems \ref{thm:dist-dmgt} and \ref{thm:dist-dmgt-filt} are for a general threshold schedule $\{\tau_t\}_{t \in [|\D|]}$. One particular threshold schedule of common interest targets a set size pre-specified by the analyst. Ways to achieve this goal are known in the literature \cite[]{agrawal2018,badanidiyuru2014,kazemi2019,liu2021,nikolakaki2021,norouzi-fard2018}. Our results readily provide novel guarantees for distributed versions of the algorithms in \cite{agrawal2018,badanidiyuru2014,kazemi2019,liu2021,nikolakaki2021,norouzi-fard2018}, which use thresholds designed to achieve a target set size.
\\\\
2. An optimality guarantee for \texttt{DMGT} can be recovered from \eqref{eq:dist-dmgt-bound} by setting $M=1$.

\subsection{Distributed DMGT with Filtering}
\texttt{Distributed DMGT} returns a high-value union of subsets from a collection of distributed streams. If those subsets are quite similar though, a central agent can further filter them, ultimately returning a smaller set with comparable value. In this case, the dependence on $M$ is undesirable because the collective value of the agents' sets does not scale with $M$. We design \texttt{Distributed DMGT w/ Filtering} to address this problem. In particular, we show that with further filtering by a central agent, the dependence on $M$ in our value-approximation bound potentially vanishes. 

In \texttt{Distributed DMGT w/ Filtering}, there are again $M$ agents with individual data streams $\{\D^j\}_{j \in [M]}$, plus a single central agent. Each distributed agent runs \texttt{DMGT} individually on their stream and whenever they select a point, they broadcast it to the central agent. This induces a stream of points for the central agent, generated by the distributed agents' collective selection decisions. The central agent simultaneously runs \texttt{DMGT} on this induced stream and returns the selected set. For concision in Algorithm \ref{alg:dist-dmgt-filt}, we present the distributed agents' and central agent's selection steps sequentially, and defer the fully online version to Appendix \ref{apx:dist-dmgt-filt}. The theoretical guarantee in Theorem \ref{thm:dist-dmgt-filt} holds for both versions.

\begin{algorithm}
\caption{\texttt{Distributed DMGT w/ Filtering}: Distributed Dynamic Marginal Gain Thresholding w/ Filtering}\label{alg:dist-dmgt-filt}
\textbf{Input} $M$ streams of data $\{\D^j\}_{j\in[M]}$;
value function $f:2^{\D} \rightarrow \R^{+}$ (where $\D = \cup_{j \in [M]}\D^j$)
\begin{algorithmic}[1]
\State $\L = \texttt{Distributed DMGT}(\{\D_j\}_{j \in [M]},f)$
\State
$\L^{\text{central}}=\texttt{DMGT}(\L,f)$
\State {\bfseries return} $\L^{\text{central}}$.
\end{algorithmic}
\end{algorithm}

\begin{theorem}[Near optimality of \texttt{Distributed DMGT w/ Filtering}]\label{thm:dist-dmgt-filt} Suppose \text{\texttt{Distributed DMGT w/ Filtering}} runs on $M$ separate data streams $\{\D^j\}_{j \in [M]}$, each with monotone increasing, nonnegative, submodular value function $f$, and the central agent returning selected set $\L^{\text{central}}$. Let $\L = \cup_{j \in [M]}\L^j$, the union of the distributed agents' individually selected sets, and for an arbitrary set $S$, define the quantity
\begin{equation}
    \lambda(S) \triangleq
    \frac{\tau_{\min}(S)}{\tau_{\min}(S) + \tau_{\max}(S)},
\end{equation}
where $\tau_{\min}(S)$ denotes the minimum threshold value corresponding to points in $S$, and $\tau_{\max}(S)$ the maximum value. Then
\begin{align}
  &f(\L^{\text{\emph{central}}}) 
  \geq \\ 
  &\bigg[\min\bigg(1,\frac{|\L^{\text{\emph{central}}}|}{\max_{j \in [M]}|\L^{j}|}\bigg) \min\bigg(1, \frac{\min_{j \in [M]}|\L^{j}|}{|\L^{\text{\emph{central}}}|}\bigg)
  \frac{\lambda(\L^{\text{\emph{central}}})\min_{j \in [M]}\lambda(\L^j)}{\min(M,|\L^{\text{\emph{central}}}|)}\bigg]f(\tOPT(f,\D,|\L^{\text{\emph{central}}}|))\label{eq:dist-dmgt-filt}.
\end{align}
\end{theorem}
\begin{proof}
    See \ref{proof:dist-dmgt-filt}.
\end{proof}
When the central agent has selection budget $k$, \eqref{eq:dist-dmgt-filt} gives that the selected set's value is $\O(\nicefrac{1}{\min(M,k)})$-close to the value of the optimal same-sized set from the distributed streams. This is the same scaling observed in the offline distributed submodular maximization algorithms in \cite{mirzasoleiman2013}.

\subsection{Examples of Threshold Schedules}\label{sec:thresh-sched}
\paragraph{Uniform.} The $\nicefrac{\tau_{\min}}{\tau_{\min} + \tau_{\max}}$-factor in our bounds, \eqref{eq:dist-dmgt-bound}, is maximal at $\nicefrac{1}{2}$ for uniform threshold values ($\tau_{\min}=\tau_{\max}$). This shows that setting a high uniform value for the thresholds will yield a high-value set and while restricting the size of that set, thus saving selection costs. However, other threshold schedules may be more suitable for certain settings/objectives (e.g. see Table \ref{table: threshold schedules} and the next discussion of marginal cost for examples). With a well-chosen, non-uniform threshold schedule, our algorithm might return a set whose true factor of optimality is much greater than $\nicefrac{1}{2}$. This shows that while our bounds in Theorems \ref{thm:dist-dmgt} and \ref{thm:dist-dmgt-filt} give optimality guarantees for any threshold-based streaming algorithm, they are not tight for all cases.
\paragraph{Marginal Cost.} Thresholds can also be set adaptively, based on information acquired from the data stream.
One intuitive way to do this, given the motivating themes of \emph{value} and \emph{cost} in this paper, is as a sequence of marginal costs. Formally, \emph{cost} can be defined as a function $c:2^{\D} \rightarrow \R^+$ which takes in a subset of $\D$ and assigns a cost to it in the form of a real nonnegative number. For the practitioner who wants to train a classifier on the selected data, one intuitive choice for $c$ is cardinality, or a positive-constant scaling of cardinality, reflecting that larger sets are more costly to select since they contain more points which require expensive labeling. Given such a cost, we can define a sequence of thresholds. One reasonable approach is: at time step $t$, given the currently selected set $\L_{t-1}$, only select the current point $x_t$ if the marginal value of doing so outweighs the marginal cost; that is, if
\begin{equation}\label{eq:marginal-cost}
    f(\L_{t-1} \cup \{x_t\}) - f(\L_{t-1})
    \\
    > c(\L_{t-1} \cup \{x_t\}) - c(\L_{t-1}).
\end{equation}
Deploying this rule, the analyst obtains a set whose value obeys \eqref{eq:dist-dmgt-bound}, while ensuring that at every time step, they never pay more for selecting a point than the value it contributes to the set. 
\paragraph{Budget-constrained.} If an agent has a budget constraint of $k$, they can set $\tau_t = f(\{x_t\})$ for all $t$ past the point when the budget is reached. This ensures that no more than $k$ points are selected since by submodularity of $f$, it's never the case that
\begin{equation}
    f(\L_{t-1} \cup \{x_t\}) - f(\L_{t-1}) > f(\{x_t\}).
\end{equation}

\section{Experiments}\label{sec:exps}

We now deploy our method on the practical problem of class imbalance correction, with the objective of increasing prediction accuracy on large scale computer vision datasets. This example clearly highlights the gains that thoughtful selection of data provides over random selection. Our goal is to extract at minimal cost an optimal training set for a classifier, $\hat{\pi}$, from streams of data. We run two experiments in the online distributed setting to accomplish this. In Experiment 1, which implements \texttt{Distributed DMGT}, $M=3$ agents select data from distributed and differently class-imbalanced streams over multiple rounds. After each round, $\hat{\pi}$ is updated on the agents' pooled selected sets and then is used by the agents for selection decisions in the next round. Experiment 2, which implements \texttt{Distributed DMGT w/ Filtering}, is identical to Experiment 1, except we update $\hat{\pi}$ after each round on a further selected subset of the agents' pooled selected sets.

In these experiments, we will make the (not always true) assumption that class-balance correlates improved prediction accuracy, i.e. classifiers trained on class-balanced sets predict better. Consequently, we design our value function to select class-balanced subsets from the data stream. Specifically, an agent who encounters the point $x_t$ in their stream will query the label $y_t$ of that point and add $(x_t,y_t)$ to their currently labeled set $\L_{t-1}$ if
\begin{align}\label{eq:class-balance-sel-rule}
    \sum_{k \in [K]} \hat{\pi}_k(x_t)\big[\sqrt{1 + |\{(x,y) \in \L_{t-1}: y = k\}|} - 
    \sqrt{|\{(x,y) \in \L_{t-1}: y = k\}|}\big]
    >
    \tau_t,
\end{align}
where $[K]$ is the set of classes. Since the selection rule \eqref{eq:class-balance-sel-rule} depends on the labels of prior-selected points, our experiments are in the adaptive-labeling regime. We find that assuming access to past labels improves our model's accuracy and valuation of sets. Our experiments can also be run with a version of \eqref{eq:class-balance-sel-rule} that doesn't use past labels (see Supplement Section 2.2 for details).

We take $\X$ to be images from MNIST and ImageNet, and $\Y$ to be the corresponding labels for those images. For MNIST, we take $\hat{\pi}$ to be an untrained ResNet-50 classifier from the \texttt{torchvision} \cite[]{paszke2019} repository, and for ImageNet, we use the same base model, but pre-train it on SimCLR feature embeddings \cite[]{chen2020} for efficiency. For training $\hat{\pi}$, we use an SGD optimizer with learning rate $10^{-3}$, momentum $0.9$, and weight-decay $5^{-4}$. In forming our class-imbalanced data streams, we follow the same recipe for the agents: 
\begin{enumerate}
    \item Partition the set of classes $[K]$ into $\mathcal{R}$, rare classes, and $\C$, common classes. 
    For all experiments on MNIST, we set $\mathcal{R}$ to be classes $\{0,1,2,3,4\}$ and $\C$ to be classes $\{5,6,7,8,9\}$. For ImageNet, we populate $\mathcal{R}$ and $\C$ with $5$ randomly selected classes each (using the same classes for all agents).
    \item Set an imbalance parameter $\beta$ and from $\X^n \times \Y^n$, sample $\D$, an $\beta \times$-imbalanced sequence of image-label pairs. That is, $\D$ should contain $\beta$ times as many points whose $\Y$ coordinate belongs to $\C$ as to $\mathcal{R}$.
\end{enumerate}    
We run our algorithms over multiple rounds, updating $\hat{\pi}$ on the selected set after each round and initializing it with weights from the previous update. Before deploying $\hat{\pi}$ on its first stream, we warm-start train it on $1000$ points which are class-imbalanced in the same way as the stream. For each update episode, we train $\hat{\pi}$ over $200$ epochs or until training accuracy achieves $0.99$. Between rounds, we calibrate $\hat{\pi}$ with isotonic regression on held-out validation data. For this purpose, on ImageNet we use half of the validation set, and on MNIST half of the test set. The other halves are held back for prediction experiments.

After a few iterations of selection, we obtain a classifier which, having only trained on a fraction of the stream, predicts well on all classes and rare classes at test time.
\begin{figure*}[htbp]
  \centering
  \begin{subfigure}{0.49\textwidth}
    \includegraphics[width=\textwidth,scale=4]{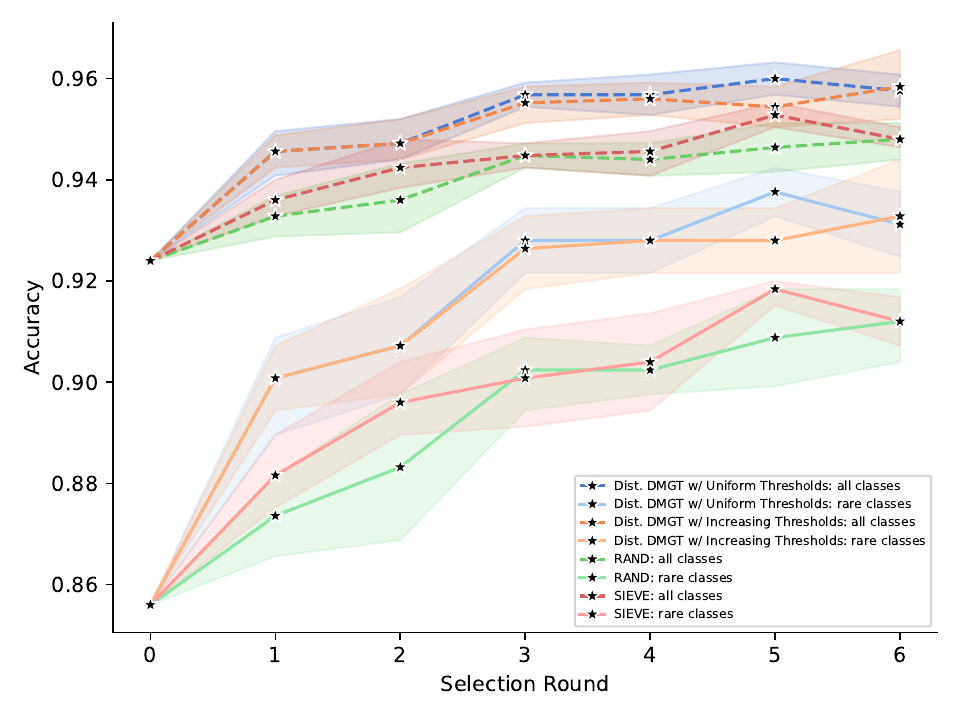}
    \subcaption{ImageNet: accuracy on all and rare classes}
    \label{fig:imnet-dist-dmgt-acc}
  \end{subfigure}
  \hfill
  \begin{subfigure}{0.49\textwidth}
    \includegraphics[width=\textwidth,scale=4]{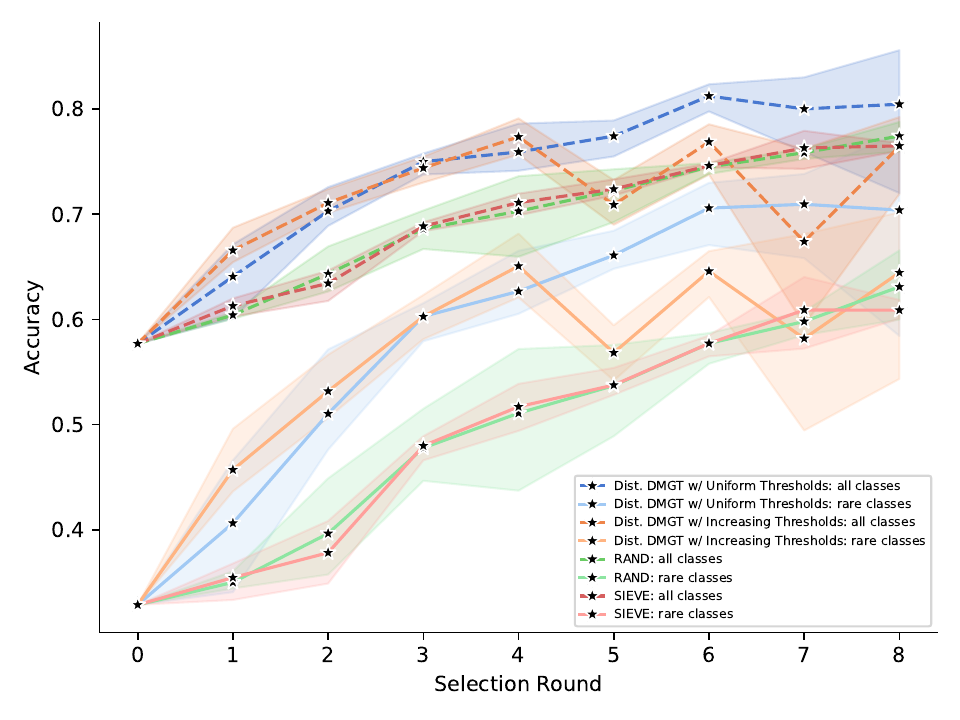}
    \subcaption{MNIST: accuracy on all and rare classes}
    \label{fig:mnist-dist-dmgt-acc}
  \end{subfigure}
  \caption{\textbf{Prediction accuracy on ImageNet and MNIST of \texttt{Distributed DMGT}.} We show prediction accuracy of \texttt{Distributed DMGT} with uniform and increasing threshold schedules vs. \texttt{RAND} and \texttt{SIEVE} after running all algorithms for $6$ selection rounds on ImageNet and $8$ selection rounds on MNIST. Figures show $95\%$ confidence intervals around the mean accuracy value over $5$ random permutations of $\D$ for ImageNet and $3$ random permutations of $\D$ for MNIST. The legend applies to both plots.}
  \label{fig:dist-dmgt-acc}
\end{figure*}

We run both \texttt{Distributed DMGT} (Experiment 1) and \texttt{Distributed DMGT w/ Filtering} (Experiment 2) with uniform and monotone increasing thresholds. We compare the performances of these two threshold schedules against two baselines. The first baseline is a random selection procedure \texttt{RAND} which randomly selects a subset from the stream. The second baseline is a canonical online threshold-based algorithm, SIEVE-STREAMING, proposed by \cite{badanidiyuru2014}. We chose  SIEVE-STREAMING because 1) it was the first online submodular maximization algorithm competitive with the optimal algorithm for offline submodular maximization (SIEVE is $\nicefrac{1}{2}-\epsilon$-optimal, for a tune-able error term $\epsilon$, while the best known offline result is $\nicefrac{1}{e}$-optimal); and 2) many subsequent well-performing online submodular maximization algorithms are derivatives of SIEVE-STREAMING.

All code for our experiments is at this \href{https://github.com/mwerner28/dmgt-code}{github repo}.
\subsection{Experiment 1: Accuracy of \texttt{Distributed DMGT} on ImageNet and MNIST}\label{sec:dist-dmgt-acc-exp}

In Figure~\ref{fig:dist-dmgt-acc}, we run each of our four algorithms with three agents, whose streams each have different imbalance factors ($\beta=2, 5, \ \text{and} \ 10$). For both datasets, during each round, each of the agents receives a 500-point stream and has a budget of 250 points. In each selection round, each agent runs the algorithm on their individual stream and then sends their selected set to the global model, $\hat{\pi}$. The model updates on the agents' respective pooled selected sets. The updated model is then broadcast back to the agents to be used for their selection decisions in the next round. At each round, Figure~\ref{fig:dist-dmgt-acc} compares the prediction accuracy of $\hat{\pi}$ trained on on the agents' pooled selected sets. We see that models trained on the \texttt{Distributed DMGT} sets with both types of thresholds generally outperform \texttt{SIEVE} sets. \texttt{Distributed DMGT} sets with uniform thresholds outperform \texttt{SIEVE} and \texttt{RAND} sets by $\sim10\%$ on MNIST rare classes and $\sim2\%$ on ImageNet rare classes.
\subsection{Experiment 2: Accuracy of \texttt{Distributed DMGT w/ Filtering} on ImageNet and MNIST}\label{sec:filt-dist-dmgt-acc-exp}
The setup for Experiment 2 is identical to that of Experiment 1. However, instead of updating $\hat{\pi}$ on the agents' pooled selected sets, a central agent with a budget of 250 points on ImageNet and 500 points on MNIST runs the selection algorithms once more on the agents' pooled selected sets and updates $\hat{\pi}$ on those further filtered sets. We see that models trained on the \texttt{Distributed DMGT} sets with both types of thresholds generally outperform \texttt{SIEVE}. \texttt{Distributed DMGT} sets with both types of thresholds outperform \texttt{SIEVE} and \texttt{RAND} sets by $\sim5-20\%$ on MNIST rare classes and $\sim1-5\%$ on ImageNet rare classes.

\begin{figure*}[htbp]
  \centering
  \begin{subfigure}{0.49\textwidth}
    \includegraphics[width=\textwidth,scale=4]{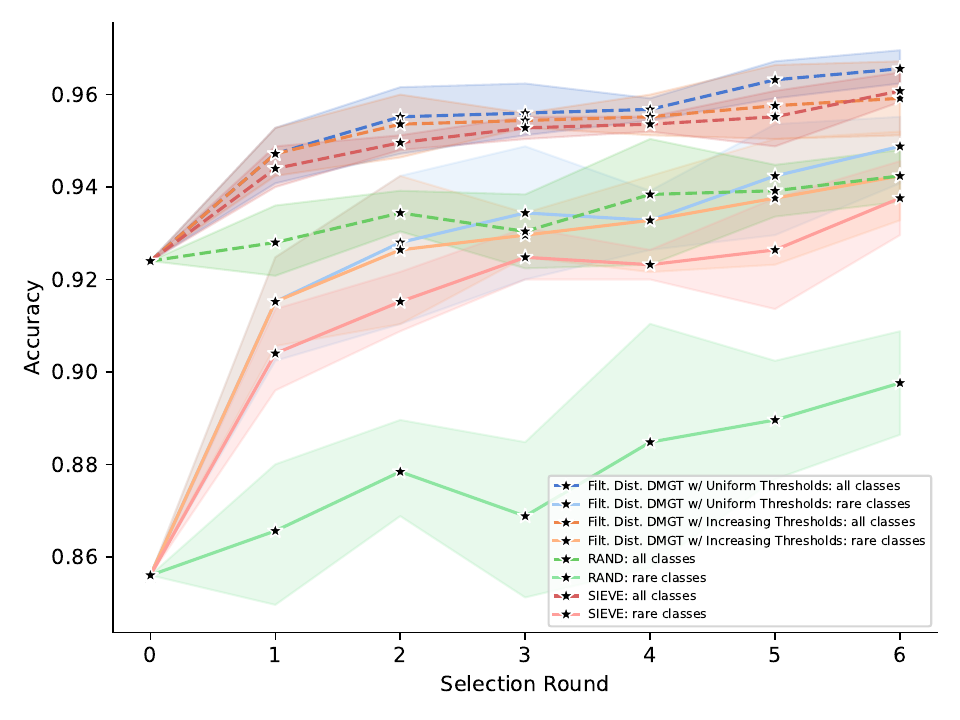}\label{fig:imnet-filt-dist-dmgt-acc}
    \subcaption{ImageNet: accuracy on all and rare classes}
  \end{subfigure}
  \hfill
  \begin{subfigure}{0.49\textwidth}
    \includegraphics[width=\textwidth,scale=4]{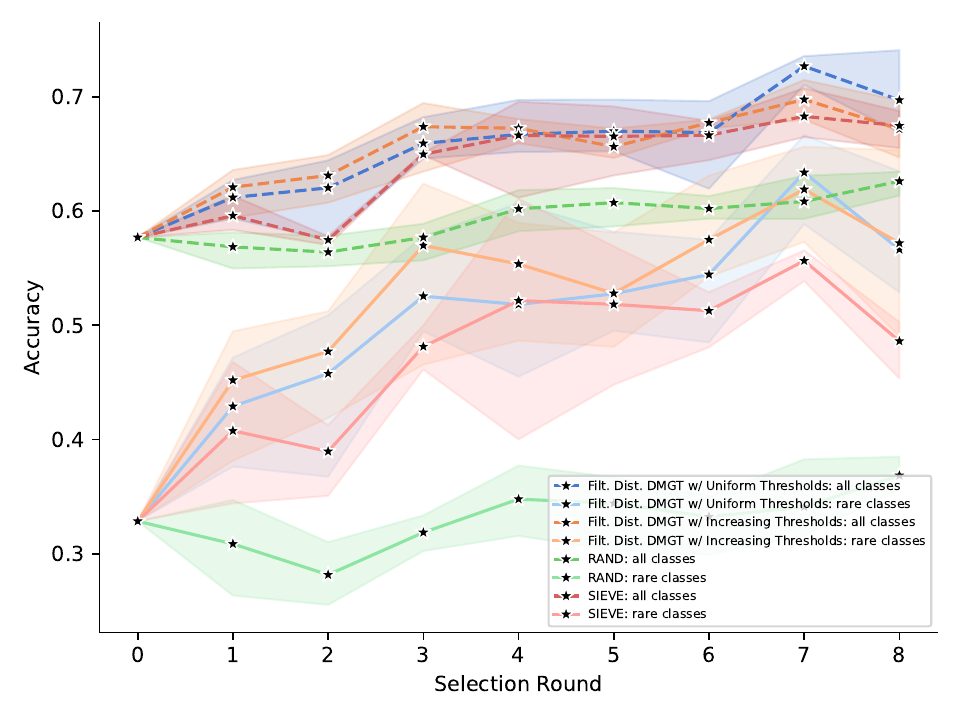}\label{fig:mnist-filt-dist-dmgt-acc}
    \subcaption{MNIST: accuracy on all and rare classes}
  \end{subfigure}
  \caption{\textbf{Prediction accuracy on ImageNet and MNIST of \texttt{Distributed DMGT w/ Filtering}.} We show prediction accuracy of \texttt{Distributed DMGT w/ Filtering} with uniform and increasing threshold schedules vs. \texttt{RAND} and \texttt{SIEVE}. Figures show $95\%$ confidence intervals around the mean accuracy value over $5$ random permutations of $\D$ for ImageNet and $3$ random permutations of $\D$ for MNIST.}
  \label{fig:filt-dist-dmgt-acc}
\end{figure*}

\subsection{Experiment 3: Class Balance using \texttt{Distributed DMGT w/ Filtering} on ImageNet and MNIST}\label{sec:class-bal-exp}
Figure \ref{fig:class-balance-filt-dmgt} gives more nuanced information on the types of sets \texttt{Distributed DMGT w/ Filtering} selects when using uniform vs. increasing thresholds. We see that the \texttt{SIEVE} and \texttt{RAND} sets continue to be class-imbalanced, whereas both \texttt{Distributed DMGT w/ Filtering} sets approach class-balance. Furthermore, \texttt{Distributed DMGT w/ Filtering} sets that use increasing thresholds are much smaller than \texttt{Distributed DMGT w/ Filtering} sets that use uniform thresholds, while the accuracy of the models trained on both sets is comparable (Figure~\ref{fig:filt-dist-dmgt-acc}). This gives empirical evidence in favor of well-chosen threshold schedules, even though our theory gives the tightest lower bound for uniform threshold-schedules. See Supplement Section 2.3 for more detailed discussion of this experiment. 
\begin{figure}
  \centering
  \begin{subfigure}{0.4\textwidth}
    \includegraphics[width=\textwidth,scale=4]{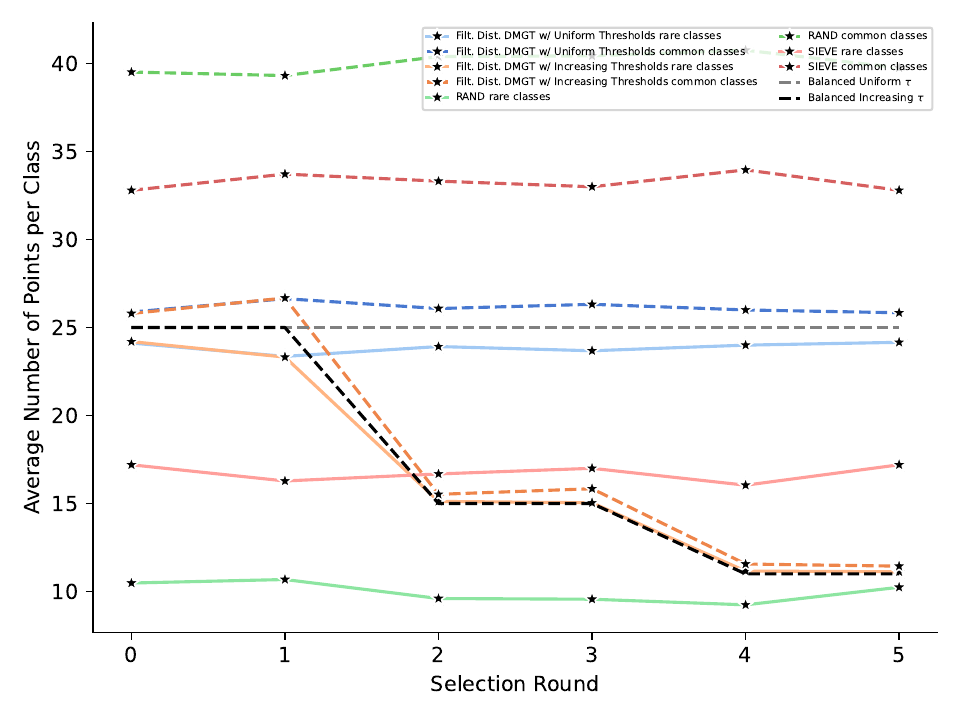}\label{fig:imnet-class-balance-filt-dist}
  \end{subfigure}
  \begin{subfigure}{0.4\textwidth}
    \includegraphics[width=\textwidth,scale=4]{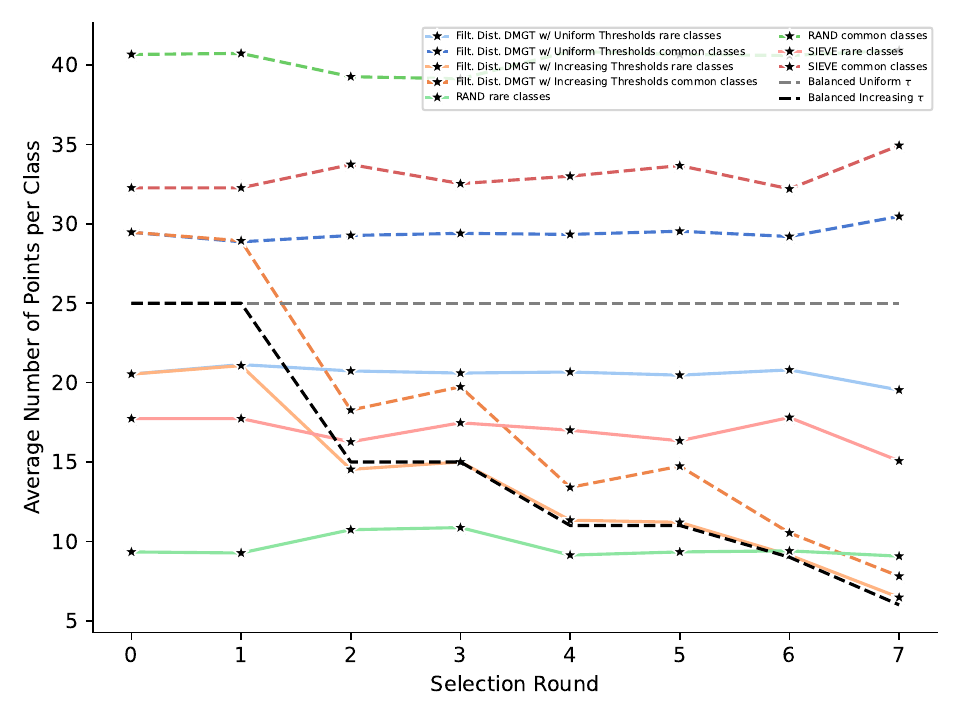}\label{fig:mnist-class-balance-filt-dist}
    \subcaption{Class-balance on all and rare classes -- ImageNet (upper plot), MNIST (lower plot)}
  \end{subfigure}
  \caption{\textbf{Class Balance on ImageNet and MNIST of \texttt{Distributed DMGT w/ Filtering}.} We report class balance convergence of \texttt{Distributed DMGT w/ Filtering} with uniform and increasing threshold schedules vs. \texttt{RAND} and \texttt{SIEVE}.}
  \label{fig:class-balance-filt-dmgt}
\end{figure}
\printbibliography
\newpage
\appendix
\section{Proofs}\label{apx:proofs}
\subsection{Proof of Theorem \ref{thm:dist-dmgt}}\label{proof:dist-dmgt}
\begin{proof}
For each $j\in[M]$, let $\D^j$ be the $j$'th agent's unlabeled data stream, let $\tau^j_x$ be the $j$'th agent's threshold value corresponding to the arrival of an arbitrary point $x \in \D^j$, let
\begin{equation*}
    \L^j
    =
    \{l^j_1,...,l^j_{s_j}\}
\end{equation*}
be agent $j$'s selected set by \texttt{DMGT}, let 
\begin{equation*}
    \L
    =
    \cup_{j \in [M]} \L^j,
\end{equation*}
and let
\begin{equation*}
    \V^j 
    =
    (\tOPT(f,\D,|\L|)\cap \D^j)\backslash \L^j
    =
    \{v^j_1,...,v^j_{m_j}\}.
\end{equation*}
Finally let $S_{<x}$ denote the history of an arbitrary sequence $S$ prior to the arrival of an arbitrary point $x$. Note that in our algorithms we index thresholds $\tau$ by time $t$. Since points from the stream arrive sequentially, we can equivalently index thresholds by points $x$, which is what we do in these proofs. 

Following the initial proof strategy used for Theorem 5.1 in \cite{nikolakaki2021},
\begin{align}
    0 
    &> 
    \sum_{j \in [M]}\sum_{i\in [m_j]}f(\L^j_{<v^j_i} \cup \{v^j_i\}) - f(\L^j_{<v^j_i}) - \tau^j_{v^j_i}
    && \text{step 5 of Alg. \ref{alg:dmgt}; def. of $\V^j$}
    \\&\geq
    \sum_{j \in [M]}\sum_{i\in [m_j]}f(\L^j \cup \V^j_{<v^j_i} \cup \{v^j_i\}) - f(\L^j \cup \V^j_{<v^j_i}) - \tau^j_{v^j_i}
    && \text{submodularity of $f$}
    \\&\geq
    \sum_{j \in [M]}f(\L^j \cup \V^j) - f(\L^j) - \tau_{\max}|\V^j|
    && \text{$\tau_{\max} \geq \tau_x \ \forall x \in [\D]$; telescope sum}
    \\&\geq
    \begin{multlined}[t][10.5cm]
        f(\cup_{j\in[M]} \L^j \cup \V^j) - \tau_{\max}(|\tOPT(f,\D,|\L|)| - \\ |\tOPT(f,\D,|\L|) \cap \L|) - \sum_{j \in [M]}f(\L^j)
    \end{multlined}
    && \text{subadditivity of $f$; def. of $\V^j$}
    \\&\geq
    \begin{multlined}[t][10.5cm]
        f(\tOPT(f,\D,|\L|)) - \tau_{\max}(|\tOPT(f,\D,|\L|)| - \\ |\tOPT(f,\D,|\L|) \cap \L|) - \sum_{j \in [M]}f(\L^j)\label{eq:dmgt-pf-bnd-1}.
    \end{multlined}
    && \text{defs. of $\L^j$, $\V^j$}
\end{align}
For all $j\in[M]$, we have
\begin{align}
    f(\L^j)
    &=
    \bigg[\sum_{i \in [s_j]}f(\L^j_{<l^j_i}\cup \{l^j_i\}) - f(\L^j_{<l^j_i})\bigg]
    \\&>
    \bigg[\sum_{i \in [s_j]}\tau^j_{l^j_i}\bigg]
    &&\text{sel. rule of Alg. \ref{alg:dmgt}}
    \\&\geq
    \tau_{\min}|\L^j|\label{eq:dmgt-pf-bnd-2}.
    &&\text{$\tau_{\min} \leq \tau_x \ \forall x \in [\D]$}
\end{align}
Combining~\eqref{eq:dmgt-pf-bnd-1} and~\eqref{eq:dmgt-pf-bnd-2},
\begin{align}
    \sum_{j\in[M]}f(\L^j) 
    &\geq
    \tau_{\min}|\L|
    &&\text{by \eqref{eq:dmgt-pf-bnd-2}}
    \\&=
    \frac{\tau_{\min}}{\tau_{\max}}\tau_{\max}|\tOPT(f,\D,|\L|)|
    \\&\geq
    \frac{\tau_{\min}}{\tau_{\max}}\bigg[f(\tOPT(f,\D,|\L|)) + \tau_{\max}|\tOPT(f,\D,|\L|) \cap \L| - \sum_{j \in [M]}f(\L^j)\bigg]
    &&\text{by \eqref{eq:dmgt-pf-bnd-1}},
\end{align}
which implies that
\begin{equation}\label{eq:dmgt-pf-bnd-3}
    \sum_{j\in[M]}f(\L^j)
    \geq
    \frac{\tau_{\min}}{\tau_{\min} + \tau_{\max}}f(\tOPT(f,\D,|\L|)) + \frac{\tau_{\min}\tau_{\max}}{\tau_{\min} + \tau_{\max}} |\tOPT(f,\D,|\L|) \cap \L|.
\end{equation}
Finally, let $j^* = \argmax_{j\in[M]}f(\L^j)$. Then
\begin{align}
    f(\L)
    &\geq
    f(\L^{j^*}) 
    &&\text{inc. mon. of $f$}
    \\&\geq 
    \frac{1}{M}\sum_{j\in[M]}f(\L^j)
    &&\text{def. of $j^*$}
    \\&\geq
    \frac{\tau_{\min}}{M(\tau_{\min} + \tau_{\max})}f(\tOPT(f,\D,|\L|)) + \frac{\tau_{\min}\tau_{\max}}{M(\tau_{\min} + \tau_{\max})} |\tOPT(f,\D,|\L|) \cap \L|,
    &&\text{by \eqref{eq:dmgt-pf-bnd-3}}
\end{align}
concluding the proof.
\end{proof}

\subsection{Proof of Theorem \ref{thm:dist-dmgt-filt}}\label{proof:dist-dmgt-filt}
\begin{proof}
Define
\begin{itemize}
    \item $e^*=\argmax_{e \in \tOPT(f,\D,|\L^{\text{central}}|)}f(\{e\})$
    \item $\D^*$: element of $\{\D_1,...,\D_M\}$ which contains $e^*$
    \item $\L^*$: element of $\{\L_1,...,\L_M\}$ selected from $\D^*$
    \item $\L_{\max} = \argmax_{j \in M}f(\L^j)$
    \item $\L=\cup_{j \in [M]}\L^j$.
\end{itemize}
For ease, we assume in this proof that the ratio of large:small set cardinalities are whole numbers (i.e. for any $A$, $B$ such that $|A| \geq |B|$, $\nicefrac{|A|}{|B|} \in \mathbbm{N}$). 
\begin{align}
    f(\L^{\text{central}})
    &\geq
    \lambda(\L^{\text{central}})f(\tOPT(f,\L,|\L^{\text{central}}|))
    && \text{by Theorem \ref{thm:dist-dmgt}}
    \\
    &\geq
    \min\bigg(1,\frac{|\L^{\text{central}}|}{|\L_{\max}|}\bigg)
    \lambda(\L^{\text{central}})f(\L_{\max}) \\
    &\geq
    \min\bigg(1,\frac{|\L^{\text{central}}|}{|\L_{\max}|}\bigg)
    \frac{\lambda(\L^{\text{central}})}{M}\sum_{j \in [M]}f(\L^j)\\
    &\geq
    \min\bigg(1,\frac{|\L^{\text{central}}|}{|\L_{\max}|}\bigg)
    \frac{\lambda(\L^{\text{central}})}{M}\sum_{j \in [M]}\lambda(\L^j)f(\tOPT(f,\D^j,|\L^j|))
    && \text{by Theorem \ref{thm:dist-dmgt}}
    \\
    &\geq
    \min\bigg(1,\frac{|\L^{\text{central}}|}{|\L_{\max}|}\bigg)
    \frac{\lambda(\L^{\text{central}})}{M}\sum_{j \in [M]}\lambda(\L^j)f(\tOPT(f,\D^j,\min_{j \in [M]}|\L^j|))\\
    &\geq
    \min\bigg(1,\frac{|\L^{\text{central}}|}{|\L_{\max}|}\bigg)
    \frac{\lambda(\L^{\text{central}})}{M}\sum_{j \in [M]}\lambda(\L^j)f(\tOPT(f,\D,\min_{j \in [M]}|\L^j|)\cap \D^j)\\
    &\geq
    \min\bigg(1,\frac{|\L^{\text{central}}|}{|\L_{\max}|}\bigg)
    \frac{\lambda(\L^{\text{central}})\min_{j \in [M]}\lambda(\L^j)}{M}f(\text{OPT}(f,\D,\min_{j \in [M]}|\L_j|))\\
    &\geq
    \min\bigg(1,\frac{|\L^{\text{central}}|}{|\L_{\max}|}\bigg) \min\bigg(1, \frac{\min_{j \in [M]}|\L^j|}{|\L^{\text{central}}|}\bigg) \\
    &\phantom{{}=1}\cdot
    \frac{\lambda(\L^{\text{central}})\min_{j \in [M]}\lambda(\L^j)}{M}f(\tOPT(f,\D,|\L^{\text{central}}|))\\
    &\geq
    \min\bigg(1,\frac{|\L^{\text{central}}|}{\max_{j \in [M]}|\L^j|}\bigg) \min\bigg(1, \frac{\min_{j \in [M]}|\L^j|}{|\L^{\text{central}}|}\bigg) \\
    &\phantom{{}=1}\cdot
    \frac{\lambda(\L^{\text{central}})\min_{j \in [M]}\lambda(\L^j)}{M}f(\tOPT(f,\D,|\L^{\text{central}}|))\label{eq:dist-dmgt-filt-M-bound}.
\end{align}

Also,
\begin{align}
    f(\L^{\text{central}})
    &\geq
    \lambda(\L^{\text{central}})f(\tOPT(f,\L,|\L_{\text{cent}}|)) 
    && \text{by Theorem \ref{thm:dist-dmgt}}
    \\
    &\geq
    \min\bigg(1,\frac{|\L^{\text{central}}|}{|\L^*|}\bigg)\lambda(\L^{\text{central}})f(\L^*)\\
    &\geq
    \min\bigg(1,\frac{|\L^{\text{central}}|}{|\L^*|}\bigg)\lambda(\L^{\text{central}})\lambda(\L^*)f(\tOPT(f,\D^*,|\L^*|))
    && \text{by Theorem \ref{thm:dist-dmgt}}
    \\
    &\geq
    \min\bigg(1,\frac{|\L^{\text{central}}|}{|\L^*|}\bigg)\lambda(\L^{\text{central}})\lambda(\L^*)f(\{\argmax_{e \in \D^*}f(e)\})\\
    &\geq
    \min\bigg(1,\frac{|\L^{\text{central}}|}{|\L^*|}\bigg)\lambda(\L^{\text{central}})\lambda(\L^*)f(\{e^*\})\\
    &\geq
    \min\bigg(1,\frac{|\L^{\text{central}}|}{|\L^*|}\bigg)\frac{\lambda(\L^{\text{central}})\lambda(\L^*)}{|\L^{\text{central}}|}f(\tOPT(f,\D,|\L^{\text{central}}|))\\
    &\geq
    \min\bigg(1,\frac{|\L^{\text{central}}|}{|\L^*|}\bigg)\frac{\lambda(\L^{\text{central}})\min_{j \in [M]}\lambda(\L_j)}{|\L^{\text{central}}|}f(\tOPT(f,\D,|\L^{\text{central}}|))\\
    &\geq
    \min\bigg(1,\frac{|\L^{\text{central}}|}{\max_{j \in [M]}|\L_j|}\bigg)\frac{\lambda(\L^{\text{central}})\min_{j \in [M]}\lambda(\L_j)}{|\L^{\text{central}}|}f(\tOPT(f,\D,|\L^{\text{central}}|)).
    \label{eq:dist-dmgt-filt-k-bound}
\end{align}
Combining \eqref{eq:dist-dmgt-filt-M-bound} and \eqref{eq:dist-dmgt-filt-k-bound} gives
\begin{equation}
  f(\L^{\text{central}}) \geq
  \min\bigg(1,\frac{|\L^{\text{central}}|}{\max_{j \in [M]}|\L^{j}|}\bigg) \min\bigg(1, \frac{\min_{j \in [M]}|\L^{j}|}{|\L^{\text{central}}|}\bigg)\frac{\lambda(\L^{\text{central}})\min_{j \in [M]}\lambda(\L^j)}{\min(M,|\L^{\text{central}}|)}f(\tOPT(f,\D,|\L^{\text{central}}|)),
\end{equation}
concluding the proof.
\end{proof}

\section{Supplementary material for class balance experiments}\label{apx:batch-dmgt}
\subsection{\texttt{Batch-DMGT}}
Our experiments require that the value function $f$ change between rounds. Since our formal guarantees for the algorithms so far in the paper presuppose a fixed value function, the guarantees do not immediately hold for this example. We now state an analog to \texttt{DMGT} for a changing value function, along with associated guarantees. 
\begin{algorithm}
    \caption{\texttt{Batch-DMGT}}
    \label{alg:3}
    \textbf{Input} Number of batches $B$; for each batch: data stream $\D^b \subset \X$, value function $f^b:2^{\D^b} \rightarrow \R^+$
\begin{algorithmic}[1]
    \State Set currently selected points $\L=\emptyset$.
    \For {$b=1,..,B$} 
    \State $\L^b \leftarrow \texttt{DMGT}(\D^b,f^b)$
    \State $\L\leftarrow \L \cup \L^b$
    \EndFor
    \State {\bfseries return} $\L$
\end{algorithmic}
\end{algorithm}

\begin{cor}[Near-optimality of \texttt{Batch-DMGT}]\label{cor:batch-dmgt}
Suppose \text{\texttt{Batch-DMGT}} runs on $B$ arbitrary data streams $\{\D^b\}_{b\in[B]}$ with accompanying value functions $\{f^b\}_{b\in[B]}$, and constructs per-batch threshold sets $\{\T^b\}_{b \in [B]}$. Then for each $b \in [B]$
\begin{equation}\label{eq:batch-dmgt-per-batch-bound}
    f^b(\L^b)
    \geq
    \frac{\tau^b_{\min}}{\tau^b_{\min} + \tau^b_{\max}}f(\tOPT(f_b,\D^b,|\L^b|)) + \frac{\tau^b_{\min}\tau^b_{\max}}{\tau^b_{\min} + \tau^b_{\max}} |\tOPT(f^b,\D^b,|\L^b|) \cap \L^b|,
\end{equation}
where $\tau^b_{\min}$ and $\tau^b_{\max}$ are the minimum and maximum elements of $\T^b$ respectively, and
\begin{equation}\label{eq:batch-dmgt-total-bound}
    f^B(\L)
    \geq
    \frac{\tau_{\min}}{B(\tau_{\min} + \tau_{\max})}f^B(\tOPT(f^B,\D,|\L|)) + \frac{\tau_{\min}\tau_{\max}}{B(\tau_{\min} + \tau_{\max})} |\tOPT(f^B,\D,|\L|) \cap \L|,
\end{equation}
where $\tau_{\min}$ and $\tau_{\max}$ are the minimum and maximum elements of $\T$ respectively.
\end{cor}

\begin{proof}
Inequality~\eqref{eq:batch-dmgt-per-batch-bound} is a direct consequence of Theorem 1, taking $M=1$. The proof for inequality~\eqref{eq:batch-dmgt-total-bound} is identical to the proof of Theorem 1, with $j$ indexing the batches $[B]$ instead of agents $[M]$.
\end{proof}

\subsection{Properties of class balance value function}\label{subapx:class-balance-function-props}

\begin{proof}[Submodularity, nonnegativity, and increasing monotonicity of $f_{\hat{\pi}}$]
We define the value of adding a new point $(x,y)$ to the currently labeled set $\L$ to be
\begin{equation}
    f_{\hat{\pi}}(\L \cup \{(x,y)\}) 
    = 
    \sum_{k \in [K]} {\hat{\pi}}_k(x) g\left(\sum_{z \in \L \cup \{(x,y)\}} {\hat{\pi}}_k(z_x)\right),
\end{equation}
where $z_x$ denotes the $x$-coordinate of point $z=(x,y)$.
Nonnegativity and increasing monotonicity of $f_{\hat{\pi}}$ are obvious. Submodularity, which we now show, follows from the concavity of $g$.

Let $S \subseteq T \subseteq \X \times \Y$ and $x \in \X \times \Y \backslash T$.
Since $g$ is a concave function, its secant lines have decreasing slope. Therefore, $\forall\  \hat{\pi}_{k\in [K]}(x) > 0$, 
\begin{equation}
    \frac{g\bigg(\sum_{z \in S \cup \{(x,y)\}} \hat{\pi}_k(z_x)\bigg) - g\bigg(\sum_{z \in S} \hat{\pi}_k(z_x)\bigg)}{\hat{\pi}_k(x)}
    \geq 
    \frac{g\bigg(\sum_{z \in T \cup \{(x,y)\}} \hat{\pi}_k(z_x)\bigg) - g\bigg(\sum_{z \in T} \hat{\pi}_k(z_x)\bigg)}{\hat{\pi}_k(x)}.
\end{equation}
It then follows, using the definition of $f_{\hat{\pi}}$ for the equalities, that
\begin{align}
    f_{\hat{\pi}}(S \cup \{(x,y)\}) - f_{\hat{\pi}}(S) 
    &= 
    \sum_{k \in [K]} \hat{\pi}_k(x)\left[g\left(\sum_{z \in S \cup \{(x,y)\}} \hat{\pi}_k(z_x)\right) - g\left(\sum_{z \in S} \hat{\pi}_k(z_x)\right)\right]\label{eq:exp-marg-gain} 
    \\&\geq
    \sum_{k \in [K]} \hat{\pi}_k(x)\left[g\left(\sum_{z \in T \cup \{(x,y)\}} \hat{\pi}_k(z_x)\right) - g\left(\sum_{z \in T} \hat{\pi}_k(z_x)\right)\right]\label{eq:dim-returns}
    \\&= 
    f_{\hat{\pi}}(T \cup \{(x,y)\}) - f_{\hat{\pi}}(T),
\end{align}
which proves submodularity of $f_{\hat{\pi}}$. (Note that if $\hat{\pi}_k(x) = 0$ for any $k\in[K]$, we simply remove that term from the sums above and the argument still holds).

Note that in the marginal gain expression in \eqref{eq:exp-marg-gain}, $f_{\hat{\pi}}(S)$ depends on the current point $\{x\}$ even though the set $S$ is independent of $\{x\}$. This is necessary in order for $f_{\hat{\pi}}$ to have a diminishing returns property in \eqref{eq:dim-returns}. Therefore, whenever we compute \emph{marginal gain} in our experiments, we use this $x$-dependent form of value, $f_{\hat{\pi}}$. However, when we only want to  compute the value of a set, independent of marginal gain from a certain point (e.g. the Sieve-Streaming algorithm in our experiments requires this), we calculate value independently of $x$. For our experimental setup, this value function would look like:
\begin{equation}
    f(S) = \sum_{k=1}^K \sqrt{|\{(x,y)\} \in S: y=k|}.
\end{equation}

Finally, in practice we observe that assuming access to labels of previously selected points when making selection decisions improves performance. Therefore, in our experiments in Section 3 of the main text, given currently labeled set $S$, we instead use the value function
\begin{equation}\label{eq:cb-exp-value-func}
    f_{\hat{\pi}}(S \cup \{(x,y)\}) 
    = 
    \sum_{k \in [K]} \hat{\pi}_k(x)g\big(1 + |\{(x,y) \in S:y=k\}|\big),
\end{equation} 
taking $g(x)=\sqrt{x}$, and label the current point $x$ if
\begin{equation*}
    \sum_{k \in [K]} \hat{\pi}_k(x)\big[g(1 + |\{(x,y) \in S: y = k\}|) - g(|\{(x,y) \in S: y = k\}|)\big]
    >
    \tau_{(x,y)}.
\end{equation*}
Using the same sequence of steps as above, \eqref{eq:cb-exp-value-func} can be shown to be submodular.
\end{proof}

\subsection{Further discussion of Experiment 3 in main text}
Using our selection rule (Eq. 4 in main text), we can directly see how uniform vs. increasing thresholds determine class balance in selected sets. For all experiments, we set the uniform thresholds to be 0.1, and the increasing thresholds to be [0.1,0.1,0.13,0.13,0.15,0.15] for the 6 ImageNet selection rounds and [0.1,0.1,0.13,0.13,0.15,0.15,0.17,2.0] for the 8 MNIST selection rounds. Let's say the models underlying our value functions are perfectly accurate (i.e. all softmax scores $\hat{\pi}_k$ are either 0 or 1). Then from the selection rule (Eq. 4 in main text), our \texttt{DMGT} algorithms should ideally select $\sim$ 25 points from each class in the uniform threshold regime (since $\sqrt{25} - \sqrt{24} \approx 0.1$). Similarly, for the increasing threshold regime, our algorithms should ideally select $\sim$ [25,25,15,15,11,11] points per class in the 6 ImageNet rounds and $\sim$ [25,25,15,15,11,11,9,6] points per class in the 8 MNIST rounds. The convergence of the \texttt{DMGT} algorithms to the black- and gray- dotted lines in Figure 3 in the main text approximates this behavior.

\subsection{Additional experimental results}
In Figure \ref{fig:dmgt-acc} we report ImageNet results on accuracy and class-balance convergence, running the subroutine of our algorithms, \texttt{DMGT}, for a single agent. In Figure \ref{fig:imnet-dmgt-class-balance-acc} we show that after running just a few (as few as 1) rounds of \texttt{DMGT}, we select a class-balanced subset from the class-imbalanced stream.

\begin{figure}[htbp]
  \centering
  \begin{subfigure}{0.49\textwidth}
    \includegraphics[width=\textwidth,scale=4]{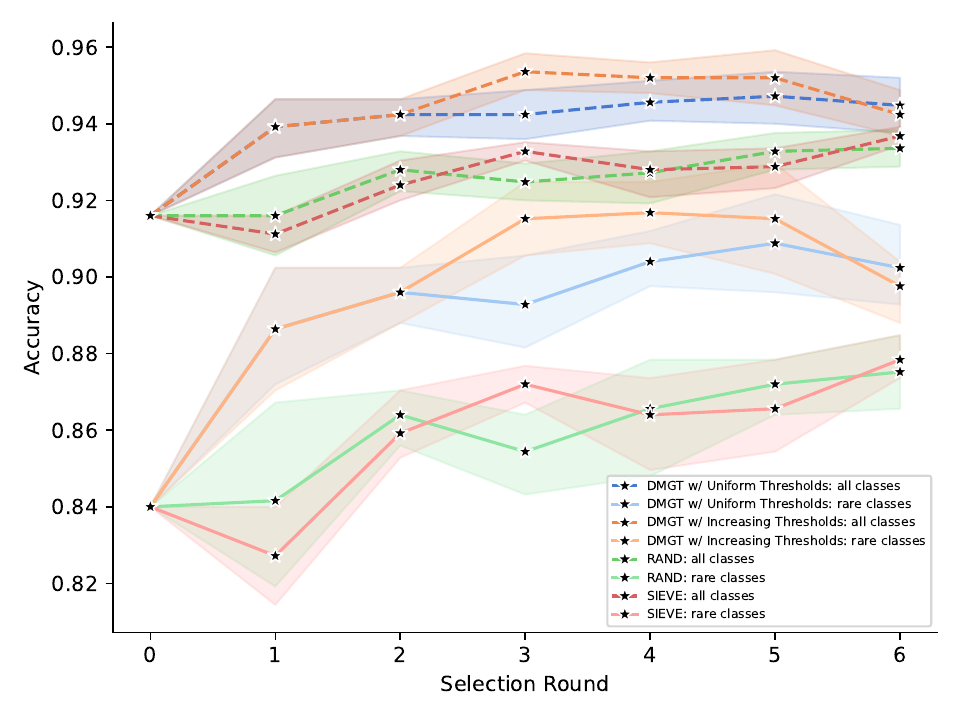}
    \subcaption{ImageNet: accuracy on all and rare classes}
    \label{fig:imnet-dmgt-acc}
  \end{subfigure}
  \hfill
  \begin{subfigure}{0.49\textwidth}
    \includegraphics[width=\textwidth,scale=4]{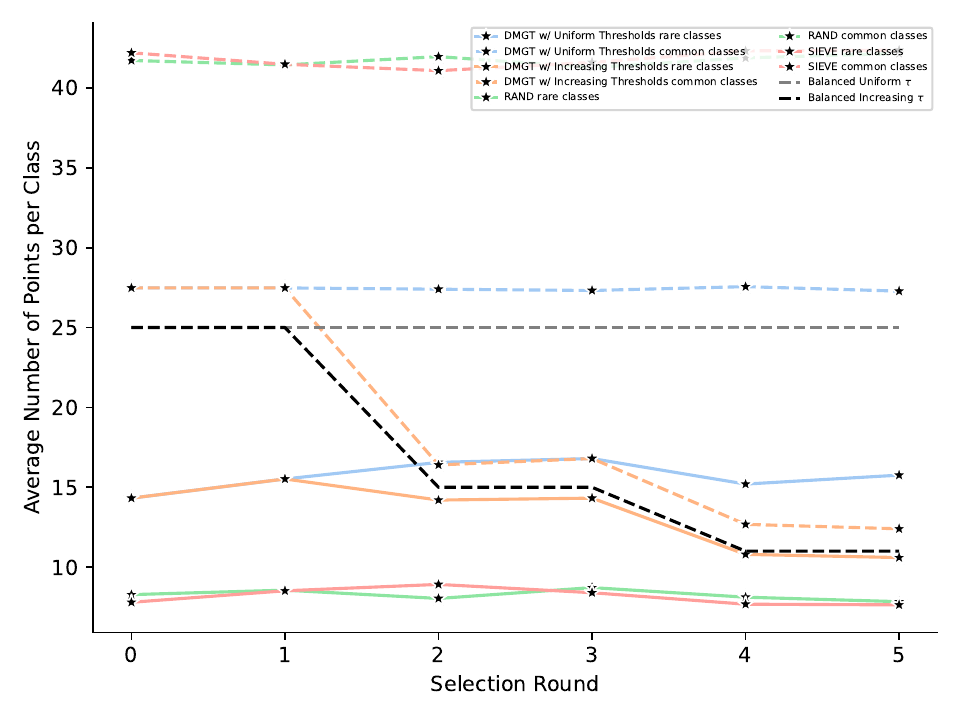}
    \subcaption{ImageNet: class balance on all and rare classes}
    \label{fig:imnet-dmgt-class-balance-acc}
  \end{subfigure}
  \caption{\textbf{Prediction accuracy on ImageNet of \texttt{DMGT}.} We show prediction accuracy of \texttt{DMGT} with uniform and increasing threshold schedules vs. \texttt{RAND} and \texttt{SIEVE} after running all algorithms for $6$ selection rounds on ImageNet. Figures show $95\%$ confidence intervals around the mean accuracy value over $5$ random permutations of $\D$ for ImageNet.}
  \label{fig:dmgt-acc}
\end{figure}

\section{Submodularity of value functions from Table \ref{table: value functions}}\label{apx:value-props}
\textbf{Class Balance}: See Section \ref{subapx:class-balance-function-props}
\newline
\newline
\textbf{Facility Location}: The Facility Location function is a special case of a weighted matroid rank function and is therefore submodular \cite{iyer2015}.
\newline
\newline
\textbf{Graph Cut}: Let $S \subseteq T$ and $z \in T^c$. Then
\begin{align}
    \sum_{z \in \S \cup \{x\}, y \in \Q}s(z,y) - \sum_{z \in \S, y \in \Q}s(z,y) 
    &=
    \sum_{y \in \Q}s(x,y) \\
    &=
    \sum_{z \in \T \cup \{x\}, y \in \Q}s(z,y) - \sum_{z \in \T, y \in \Q}s(z,y), 
\end{align}
showing this particular graph function is modular.
\section{Fully online \texttt{Distributed DMGT w/ Filtering}}
\label{apx:dist-dmgt-filt}
We give the fully online version of \texttt{Distributed DMGT w/ Filtering} in Algorithm~\ref{alg:dist-dmgt-filt-online}.
\begin{algorithm}
    \caption{\texttt{Distributed DMGT w/ Filtering}: Distributed Dynamic Marginal Gain Thresholding w/ Filtering (fully online version)}\label{alg:dist-dmgt-filt-online}
    \textbf{Input} $M$ streams of data $\{\D^j\}_{j\in[M]}$;
    value function $f:2^{\D} \rightarrow \R^{+}$ (where $\D = \cup_{j \in [M]}\D^j$)
\begin{algorithmic}[1]
    \State Set distributed agents' individual selected sets $\{\L^j=\emptyset\}_{j \in [M]}$ and central agent's selected set $\L^{\text{central}}_0=\emptyset$.
    \For {$t=1,...,|\D|$}
    \State Point $x_t$ arrives from corresponding stream $\D^i$. 
    \State Agent $i$ sets threshold $\tau^i_t = \mathcal{A}(\{x_t\} \cup \{\text{all previous points agent } i \text{ has seen}\},\tau_{1}^i,...,\tau_{t-1}^i)$.
    \If{$f(\L^i_{t-1} \cup \{x_t\}) - f(\L^i_{t-1}) > \tau^i_t$}
    \State $\L^i_t \leftarrow \L^i_{t-1} \cup \{x_t\}$
    \State Agent $i$ broadcasts $x_t$ to the central agent.
    \State Central agent sets threshold $\tau^{\text{central}}_t = \mathcal{A}(\{x_t\}\cup\{\text{all previous points central agent has seen}\}, \tau_{1}^{\text{central}},...,\tau_{t-1}^{\text{central}})$.
    \If{$f(\L^{\text{central}}_{t-1} \cup \{x_t\}) - f(\L^{\text{central}}_{t-1}) > \tau^{\text{central}}_t$}
    \State $\L^{\text{central}}_t \leftarrow \L^{\text{central}}_{t-1} \cup \{x_t\}$
    \Else 
    \State $\L^{\text{central}}_t \leftarrow \L^{\text{central}}_{t-1}$
    \EndIf
    \Else
    \State $\L^i_t \leftarrow \L^i_{t-1}$
    \State $\L^{\text{central}}_t \leftarrow \L^{\text{central}}_{t-1}$
    \EndIf
    \State $\L^{j \neq i}_t \leftarrow \L^{j \neq i}_{t-1}$
    \EndFor
    \State {\bfseries return} $\L^{\text{central}} = \L_{|\D|}^{\text{central}}$.
\end{algorithmic}
\end{algorithm}

\section{Relevant Definitions and Theorems}\label{apx:rel-defs}
\begin{definition}[Monotonicity]\label{def:mon}
A set function $f:2^{\X} \rightarrow \R$ is \emph{monotone increasing} if $\ \forall S \subseteq T$ with $S,T \subseteq \X$
\begin{equation}
    f(S) \leq f(T)
\end{equation}
and \emph{monotone decreasing} if $\ \forall S\subseteq T$ with $S,T \subseteq \X$
\begin{equation}
    f(S) \geq f(T)
\end{equation}
\end{definition}
\begin{definition}[Sub-additivity]\label{def:subadd}
A set function $f:2^{\X} \rightarrow \R$ is \emph{sub-additive} if $\ \forall S,T \subseteq \X$
\begin{equation}
    f(S \cup T)
    \leq
    f(S) + f(T)
\end{equation}
\end{definition}
\begin{definition}[Equivalent definition of submodularity]\label{def:submod-alt}
A set function $f:2^{\X} \rightarrow \R$ is \emph{submodular} if $\ \forall S,T \subseteq \X$
\begin{equation}
    f(S) + f(T)
    \geq
    f(S \cup T) + f(S \cap T)
\end{equation}
\end{definition}
\begin{theorem}[Sub-additivity of nonnegative submodular functions]\label{thm:subadd of submod}
Nonnegative submodular functions are sub-additive.
\end{theorem}
\begin{proof}[Proof of Theorem~\ref{thm:subadd of submod}]
Let $f:2^{\X} \rightarrow \R^+$ be a nonnegative submodular function. Then, by Definition~\ref{def:submod-alt}, $\ \forall S,T \subseteq \X$, we have
\begin{equation}
    f(S) + f(T)
    \geq
    f(S \cup T) + f(S \cap T)
    \geq 
    f(S \cup T).
\end{equation}
\end{proof}
\end{document}